\def\year{2016}\relax
\documentclass[letterpaper]{article}
\usepackage{aaai16}
\usepackage{times}
\usepackage{helvet}
\usepackage{courier}
\usepackage{color}
\usepackage{url}
\usepackage{amsmath} 
\usepackage{amssymb}
\usepackage{amsthm}
\usepackage{array}
\usepackage{wrapfig}
\usepackage{algorithm}
\usepackage{algorithmic}
\usepackage[pdftex]{graphicx}
\newtheorem{theorem}{Theorem}
\newtheorem{lemma}{Lemma}
\newtheorem{case}{Case}
\usepackage{epigraph}

\usepackage{bm}
\newcommand{\xSrc}{\vec{x}}
\newcommand{\ySrc}{y}
\newcommand{\xTar}{\vec{u}}
\newcommand{\wSrc}{\vec{w}}

\title{Return of Frustratingly Easy Domain Adaptation}
\author{
Baochen Sun\\
Department of Computer Science\\
University of Massachusetts Lowell\\
Lowell, MA 01854, USA \\
\texttt{bsun@cs.uml.edu} \\
\And
Jiashi Feng\\
Department of EECS, UC Berkeley,\\
USA \& Department of ECE, National\\
University of Singapore,  Singapore\\
\texttt{elefjia@nus.edu.sg} \\
\And
Kate Saenko\\
Department of Computer Science\\
University of Massachusetts Lowell\\
Lowell, MA 01854, USA \\
\texttt{saenko@cs.uml.edu} \\
}

\begin{document}
\nocopyright 
\frenchspacing
\maketitle

\begin{abstract}
Unlike human learning, machine learning often fails to handle changes between training (source) and test (target) input distributions. Such domain shifts, common in practical scenarios, severely damage the performance of conventional machine learning methods. 
Supervised domain adaptation methods have been proposed for the case when the target data have labels, including some that perform very well despite being ``frustratingly easy'' to implement. However, in practice, the target domain is often unlabeled, requiring unsupervised adaptation.
We propose a simple, effective, and efficient method for unsupervised domain adaptation called CORrelation ALignment (CORAL). CORAL minimizes domain shift by aligning the second-order statistics of source and target distributions, without requiring any target labels. 
Even though it is extraordinarily simple--it can be implemented in four lines of Matlab code--CORAL performs remarkably well in extensive evaluations on standard benchmark datasets.
\end{abstract}

\epigraph{``Everything should be made as simple as possible, but not simpler."}{\textit{Albert Einstein}}
\section{Introduction}
\label{sec:intro}
Machine learning is very different from human learning. Humans are able to learn from very few labeled examples and apply the learned knowledge to new examples in novel conditions. In contrast, supervised machine learning methods only perform well when the given extensive labeled data are from the same distribution as the test distribution. Both theoretical \cite{bendavid,Blitzer07Biographies} and practical results~\cite{saenko2010adapting,efros-cvpr11} have shown that the test error of supervised methods generally increases in proportion to the ``difference'' between the distributions of training and test examples. For example, \citeauthor{decaf} \shortcite{decaf} showed that even state-of-the-art Deep Convolutional Neural Network features learned on a  dataset of $1.2M$ images are susceptible to domain shift.
Addressing domain shift is undoubtedly critical for successfully applying machine learning methods in real world applications.

To compensate for the degradation in performance due to domain shift, 
many domain adaptation algorithms have been developed, most of which assume that some labeled examples in the target domain are provided to learn the proper model adaptation. \citeauthor{daume} \shortcite{daume} 
proposed a supervised domain adaptation approach notable for its extreme simplicity: it merely changes the features by making domain-specific and common copies, then trains a supervised classifier on the new features from both domains. The method performs very well, yet is ``frustratingly easy'' to implement. However, it cannot be applied in the situations where the target domain is unlabeled, which unfortunately are quite common in practice.

In this work, we present a ``frustratingly easy'' \emph{unsupervised} domain adaptation method called CORrelation ALignment (CORAL). CORAL aligns the input feature distributions of the source and target domains by exploring their second-order statistics. More concretely, CORAL aligns the distributions by re-coloring whitened source features with the covariance of the target distribution. CORAL is simple and  efficient, as the only computations it needs are (1) computing covariance statistics in each domain and (2) applying the whitening and re-coloring linear transformation to the source features. Then, supervised learning proceeds as usual--training a classifier on the transformed source features. 

\begin{figure}[t]
\centering
\includegraphics[width=\linewidth]{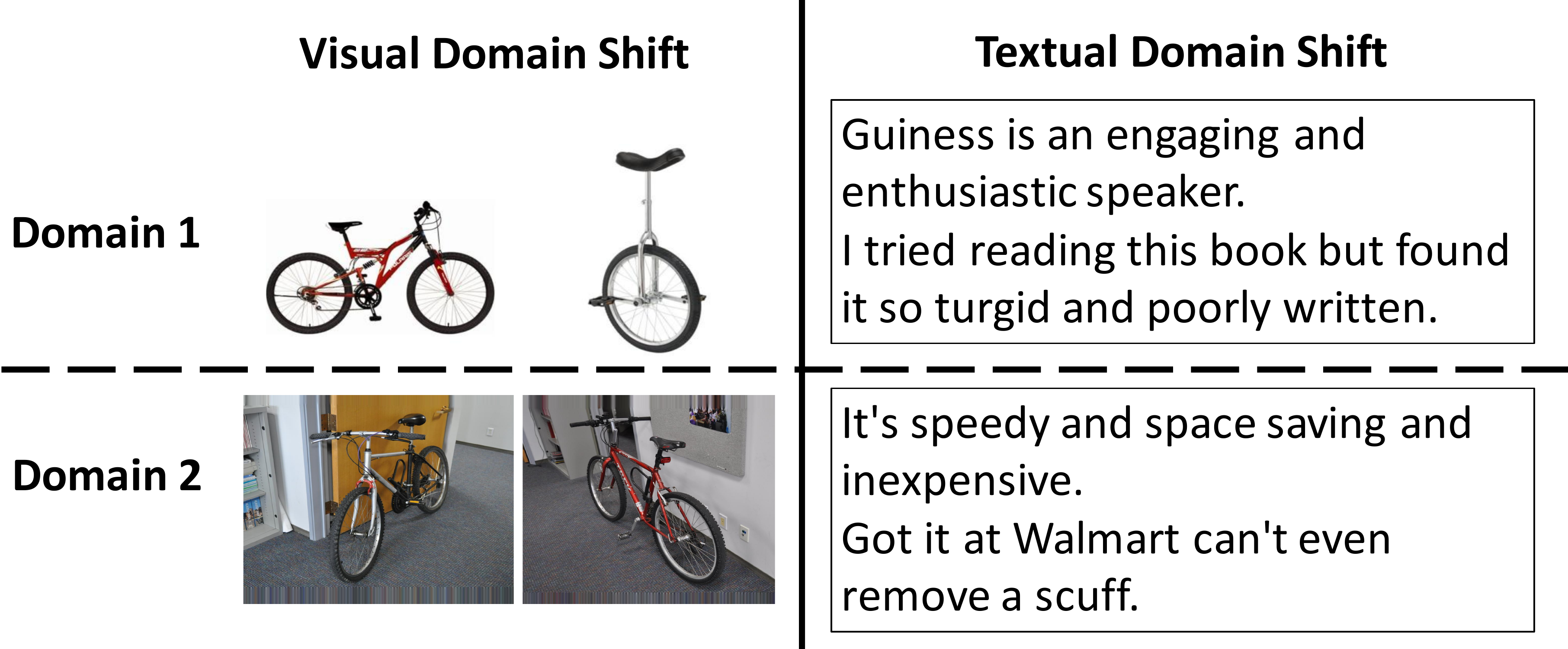}
\vspace{-0.2in}
\caption{\small Two Domain Shift Scenarios: object recognition across visual domains (left) and sentiment prediction across text domains (right). When data distributions differ across domains, applying classifiers trained on one domain directly to another domain is likely to cause a significant performance drop. }
\label{fig:shift}
\vspace{-0.1in}
\end{figure}

Despite being ``frustratingly easy'', CORAL offers surprisingly good performance on standard adaptation tasks. We apply it to two tasks: object recognition and sentiment prediction (Figure~\ref{fig:shift}), and show that it outperforms many existing methods. For object recognition, we demonstrate that it works well with both standard ``flat" bag-of-words features and with state-of-the-art deep CNN features~\cite{alexnet}, 
outperforming existing methods, including recent deep CNN adaptation approaches~\cite{tzeng_arxiv15,reversegrad,dan_long15}.
The latter approaches are quite complex and expensive, requiring re-training of the network and tuning of many hyperparameters such as the structure of the hidden adaptation layers. In contrast, CORAL only needs to compute the covariance of the source and target features. 
\section{Related Work}
\label{sec:related}

Domain shift is a fundamental problem in machine learning, and has also attracted a lot of attention in the speech, natural language and vision communities~\cite{Blitzer07Biographies,gopalan-iccv11,bmvc}.
For supervised adaptation, a variety of techniques have been proposed. Some consider the source domain as a prior that regularizes the learning problem in the sparsely labeled target domain, e.g.,~\cite{yang_icdm07}. Others minimize the distance between the target and source domains, either by re-weighting the domains or by changing the feature representation according to some explicit distribution distance metric~\cite{mmd}. Some learn a transformation on features using a contrastive loss~\cite{saenko2010adapting}. 
Arguably the simplest and most prominent supervised approach is the ``frustratingly easy'' feature replication~\cite{daume}. 
Given a feature vector $\bm{x}$, it defines the augmented feature vector $\tilde{\bm{x}} = (\bm{x}; \bm{x}; \bm{0})$ for data points in the source and $\tilde{\bm{x}} = (\bm{x}; \bm{0}; \bm{x})$ for data points in the target. A classifier is then trained on augmented features. This approach is simple, however, it requires labeled target examples, which are often not available in real world applications.

Early techniques for unsupervised adaptation consisted of re-weighting the training point losses to more closely reflect those in the test distribution~\cite{jiang-zhai07,huang_nips06}. Dictionary learning methods~\cite{dict_1,dict_2} try to learn a dictionary where the difference between the source and target domain is minimized in the new representation.
Recent state-of-the-art unsupervised approaches~\cite{gopalan-iccv11,gfk,long_cvpr,Sp_CVPR15} have pursued adaptation by projecting the source and target distributions into a lower-dimensional manifold, and finding a transformation that brings the subspaces closer together. Geodesic methods find a path along the subspace manifold, and either project source and target onto points along that path~\cite{gopalan-iccv11}, or find a closed-form linear map that projects source points to target~\cite{gfk}. Alternatively, the subspaces can be aligned by computing the linear map that minimizes the Frobenius norm of the difference between them~\cite{outlooks,sasb}. However, these approaches only align the bases of the subspaces, not the distribution of the projected points. They also require expensive subspace projection and hyperparameter selection. 

Adaptive deep neural networks have recently been explored for unsupervised adaptation. DLID~\cite{chopra2013dlid} trains a joint source and target CNN architecture, but is limited to two adaptation layers. ReverseGrad~\cite{reversegrad}, DAN~\cite{dan_long15}, and DDC~\cite{tzeng_arxiv15} directly optimize the deep representation for domain invariance, using additional loss layers designed for this purpose. Training with this additional loss is costly and can be sensitive to initialization, network structure, and other optimization settings. Our approach, applied to deep features (top layer activations), achieves better or comparable performance to these more complex methods, and can be incorporated directly into the network structure.
\section{Correlation Alignment for Unsupervised Domain Adaptation}
\label{sec:methods}

\begin{figure}[t]
\centering
\includegraphics[width=\linewidth]{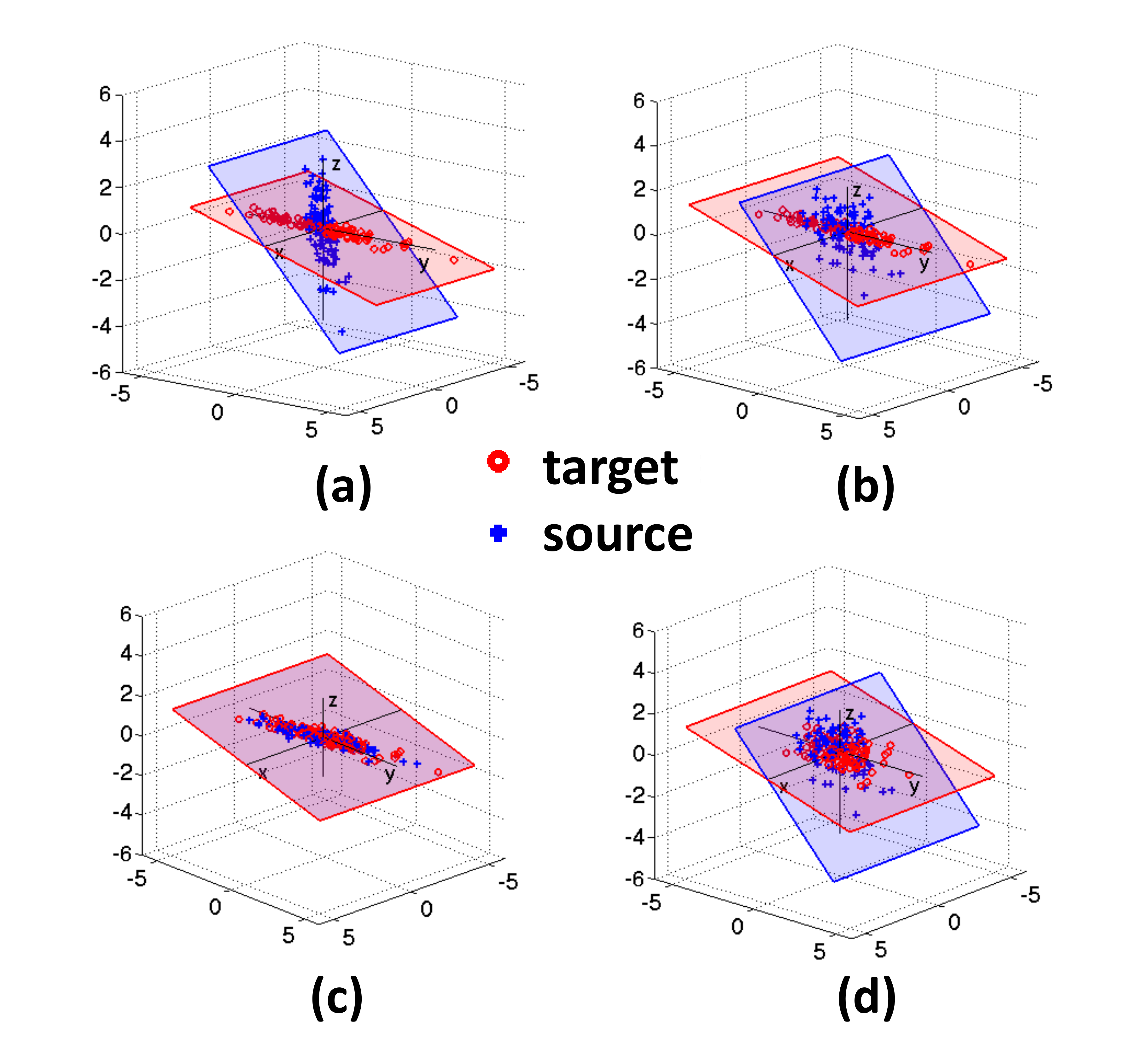}
\vspace{-0.3in}
\caption{\small \textbf{(a-c)} Illustration of CORrelation ALignment (CORAL) for Domain Adaptation: (a) The original source and target domains have different distribution covariances, despite the features being normalized to zero mean and unit standard deviation. This presents a problem for transferring classifiers trained on source to target. (b) The same two domains after source decorrelation, i.e. removing the feature correlations of the source domain. (c) Target re-correlation, adding the correlation of the target domain to the source features. After this step, the source and target distributions are well aligned and the classifier trained on the adjusted source domain is expected to work well in the target domain. \textbf{(d)} One might instead attempt to align the distributions by whitening both source and target. However, this will fail since the source and target data are likely to lie on different subspaces due to domain shift. (Best viewed in color)}
\label{fig:variance}
\vspace{-0.2in}
\end{figure}

We present an extremely simple domain adaptation method--CORrelation ALignment (CORAL)--which works by aligning the distributions of the source and target features in an unsupervised manner. 
We propose to match the distributions by aligning the second-order statistics, namely, the covariance. 

\subsection{Formulation and Derivation}
We describe our method by taking a multi-class classification problem as the running example. Suppose we are given source-domain training examples $D_S=\{\xSrc_i\}$, $\xSrc\in\mathbb{R}^D$ with labels $L_S=\{\ySrc_i\}$, $\ySrc \in\{1,...,L\}$, and target data $D_T=\{\xTar_i\}$, $\xTar \in \mathbb{R}^D$. Here both $\xSrc$ and $\xTar$ are the $D$-dimensional feature representations $\phi(I)$ of input $I$. Suppose $\mu_s,\mu_t$ and $C_{S}, C_{T}$ are the feature vector means and covariance matrices. As illustrated in Figure~\ref{fig:variance}, $\mu_t=\mu_s=0$ after feature normalization while $C_{S} \neq C_{T}$.

To minimize the distance between the second-order statistics (covariance) of the source and target features, we apply a linear transformation $A$ to the original source features and use the Frobenius norm as the matrix distance metric:
      \begin{equation}
      \begin{aligned}
      &~\underset{A}{\min} {\| C_{\hat{S}} - C_{T} \|}^2_F\\
      &= \underset{A}{\min} {\| A^{\top}C_{S}A - C_{T} \|}^2_F
      \end{aligned}
      \label{eq:obj}
      \end{equation}
where $C_{\hat{S}}$ is covariance of the transformed source features $D_sA$ and ${\|\cdot\|}^2_F$ denotes the matrix Frobenius norm. 

If $\mathrm{rank}(C_S) \geq \mathrm{rank}(C_T)$, then an analytical solution can be obtained by choosing $A$ such that $C_{\hat{S}}= C_{T}$.
However, the data typically lie on a lower dimensional manifold~\cite{outlooks,gfk,sasb}, and so the covariance matrices are likely to be low rank~\cite{who}. We derive a solution for this general case, using the following lemma.
\begin{lemma}\cite{SVT} 
\label{lemma:svt}
Let $Y$ be a real matrix of rank $r_Y$ and X a real matrix of rank at most $r$, where ${r}\leqslant{r_Y}$; let $Y={U_Y}{\Sigma_Y}{V_Y}$ be the SVD of $Y$, and ${\Sigma_{Y[1:r]}}$, $U_{Y[1:r]}$, $V_{Y[1:r]}$ be the largest $r$ singular values and the corresponding left and right singular vectors of $Y$ respectively. Then, $X^{*} = U_{Y[1:r]}{\Sigma_{Y[1:r]}}{V_{Y[1:r]}}^{\top}$ is the optimal solution to the problem of $\underset{X}\min{\| X - Y \|}^2_F$.
\end{lemma}

\begin{theorem} 
Let $\Sigma^{+}$ be the Moore-Penrose pseudoinverse of $\Sigma$, $r_{C_S}$ and $r_{C_T}$ denote the rank of $C_S$ and $C_T$ respectively.
Then, $A^{*} = U_{S}{\Sigma_S^{+}}^{\frac{1}{2}}{U_{S}}^{\top} U_{T[1:r]}{\Sigma_{T[1:r]}}^{\frac{1}{2}}{U_{T[1:r]}}^{\top}$ is the optimal solution to the problem in Equation~\eqref{eq:obj} with $r = \min(r_{C_S}, r_{C_T})$.
\end{theorem}
\begin{proof}  Since $A$ is a linear transformation, $A^{\top}C_SA$ does not increase the rank of $C_S$. Thus, $r_{C_{\hat{S}}}\leqslant{r_{C_{S}}}$. Since $C_S$ and $C_T$ are symmetric matrices, conducting SVD on $C_S$ and $C_T$ gives $C_S=U_S{\Sigma_S}{U_S}^{\top}$ and $C_T = U_T\Sigma_TU_T^\top$ respectively. We first find the optimal value of $C_{\hat{S}}$ through considering the following two  cases:
\vspace{-0.05in}
\begin{case}
$r_{C_S}>{r_{C_T}}$. The optimal solution is $C_{\hat{S}} = C_T$. Thus, $C_{\hat{S}} = U_{T}{\Sigma_{T}}{U_{T}}^{\top} = U_{T[1:r]}{\Sigma_{T[1:r]}}{U_{T[1:r]}}^{\top}$ is the optimal solution to Equation~\eqref{eq:obj} where $r = r_{C_T}$.
\end{case}
\vspace{-0.05in}
\begin{case}
$r_{C_S}\leqslant{r_{C_T}}$. Then, according to Lemma \ref{lemma:svt}, $C_{\hat{S}} = U_{T[1:r]}{\Sigma_{T[1:r]}}{U_{T[1:r]}}^{\top}$ is the optimal solution to Equation~\eqref{eq:obj} where $r = r_{C_S}$.
\end{case}
Combining the results in the above two cases yields that ${{C_{\hat{S}}}} = U_{T[1:r]}{\Sigma_{T[1:r]}}{U_{T[1:r]}}^{\top}$ is the optimal solution to Equation~\eqref{eq:obj} with $r = \min(r_{C_S}, r_{C_T})$.
We then proceed to solve for $A$ based on the above result. 
Let $C_{\hat{S}} = {A}^{\top}C_S A$, and we get:
\begin{equation*}
{A}^{\top}C_{S}{A}= U_{T[1:r]}{\Sigma_{T[1:r]}}{U_{T[1:r]}}^{\top}.
\end{equation*}
Since $C_S = U_S \Sigma_S {U_S}^\top$, we have
\begin{equation*}
{A}^{\top}U_S{\Sigma_S}{U_S}^{\top}{A}= U_{T[1:r]}{\Sigma_{T[1:r]}}{U_{T[1:r]}}^{\top}.
\end{equation*}
This gives:
\begin{equation*}
{({U_S}^{\top}A)}^{\top}{\Sigma_S}({U_S}^{\top}{A})= U_{T[1:r]}{\Sigma_{T[1:r]}}{U_{T[1:r]}}^{\top}.
\end{equation*}
Let $E ={\Sigma_S^{+}}^{\frac{1}{2}} {U_S}^{\top} {U_{T[1:r]}} {\Sigma_{T[1:r]}}^{\frac{1}{2}}{U_{T[1:r]}}^{\top}$, then the right hand side of the above equation can be re-written as ${E}^{\top}{\Sigma_S}E$. This gives
\begin{align*}
&{({U_S}^{\top}A)}^{\top}{\Sigma_S}({U_S}^{\top}{A})= {E}^{\top}{\Sigma_S}E
\end{align*}
By setting ${U_S}^{\top}A$ to $E$, we get the optimal solution of $A$ as 
\begin{equation} 
\begin{aligned}
A^{*}&={U_S}E\\
&=(U_{S}{\Sigma_S^{+}}^{\frac{1}{2}}{U_{S}}^{\top})(U_{T[1:r]}{\Sigma_{T[1:r]}}^{\frac{1}{2}}{U_{T[1:r]}}^{\top}).
 \end{aligned}
\label{eq:slu}
\end{equation}
\end{proof}

\subsection{Algorithm}
\label{subsec:algo}
We can think of transformation $A$ in this way intuitively: the first part $U_{S}{\Sigma_S^{+}}^{\frac{1}{2}}{U_{S}}^{\top}$ whitens the source data while the second part $U_{T[1:r]}{\Sigma_{T[1:r]}}^{\frac{1}{2}}{U_{T[1:r]}}^{\top}$ re-colors it with the target covariance. This is illustrated in Figure~\ref{fig:variance}(b) and Figure~\ref{fig:variance}(c) respectively. The traditional whitening is adding a small regularization parameter $\lambda$ to the diagonal elements of the covariance matrix to explicitly make it full rank and then multiply the original feature by the inverse square root (or square root for coloring) of it. The whitening and re-coloring here are slightly different from them since the data are likely to lie on a lower dimensional space and the covariance matrices could be low rank. 

In practice, for the sake of efficiency and stability, we can perform the classical whitening and coloring. This is advantageous because: (1) 
it is faster (e.g., the whole CORAL transformation takes less than one minute on a regular laptop for $D_S\in\mathbb{R}^{795\times4096}$ and $D_T\in\mathbb{R}^{2817\times4096}$) and more stable, as SVD on the original covariance matrices might not be stable and might slow to converge; (2) as illustrated in Figure~\ref{fig:sens}, the performance is similar to the analytical solution in Equation~\eqref{eq:slu} and very stable with respect to~$\lambda$. In this paper, we set~$\lambda$ to 1. The final algorithm can be written in four lines of MATLAB code as illustrated in Algorithm~\ref{alg:coral}.
 
\begin{figure}
\centering
\includegraphics[width=0.6\columnwidth]{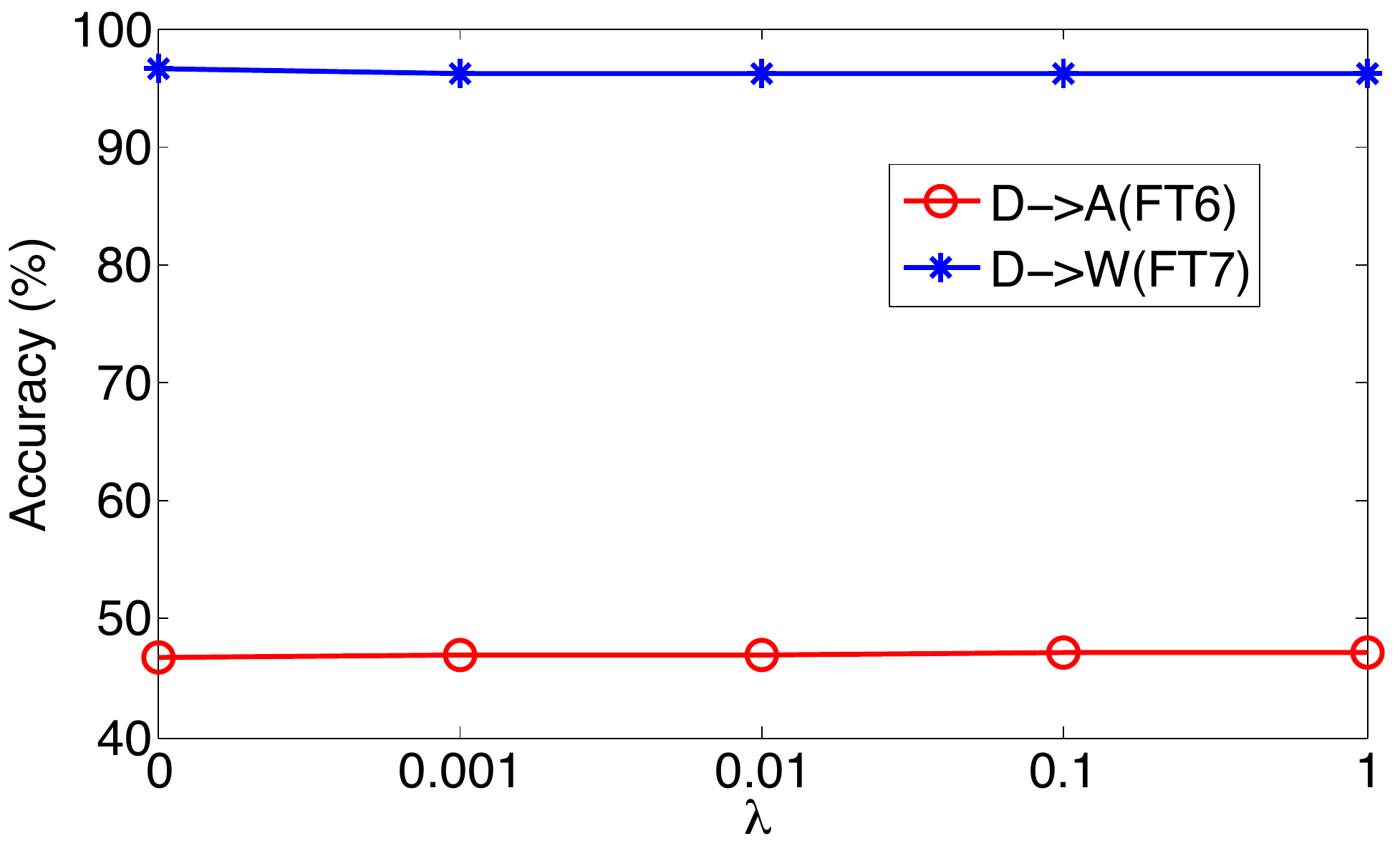}
\vspace{-0.2in}
\caption{\small Sensitivity of Covariance Regularization Parameter~$\lambda$ with~$\lambda \in$ \{0, 0.001, 0.01, 0.1, 1\}. When $\lambda = 0$, there is no regularization and we use the analytical solution in Equation~\eqref{eq:slu}. Please refer to Section~\ref{subsec:recog} for details of tasks.}
\label{fig:sens}
\end{figure}

\begin{algorithm}
\caption{CORAL for Unsupervised Domain Adaptation}
\begin{small}
\begin{algorithmic} 
\STATE \textbf{Input:} Source Data $D_S$, Target Data $D_T$
\STATE \textbf{Output:} Adjusted Source Data $D_{s}^{*}$
\STATE $C_S = cov(D_S) + eye(size(D_S, 2))$
\STATE $C_T = cov(D_T) + eye(size(D_T, 2))$
\STATE $D_S = D_S*C_S^{\frac{-1}{2}}$  ~~~~~~~~~~~~~~~~~~~~~~~~~~~~~~~~~~~~~~~~\% whitening source
\STATE $D_{S}^{*} = D_S*C_T^{\frac{1}{2}}$  ~~~~~~~~~~~~~~~~\% re-coloring with target covariance
\end{algorithmic} 
\end{small}
\label{alg:coral}
\end{algorithm}

One might instead attempt to align the distributions by whitening both source and target. As shown in Figure~\ref{fig:variance}(d), this will fail as the source and target data are likely to lie on different subspaces due to domain shift. An alternative approach would be whitening the target and then re-coloring it with the source covariance. However, as demonstrated in~\cite{outlooks,sasb} and our experiments, transforming data from source to target space gives better performance. 
This might be due to the fact that by transforming the source to target space the classifier was trained using both the label information from the source and the unlabelled structure from the target.

After CORAL transforms the source features to the target space, a classifier $f_{\wSrc}$ parametrized by $\wSrc$ can be trained on the adjusted source features and directly applied to target features. For a linear classifier 
$f_{\wSrc}(I) =  \wSrc^T \phi(I)$,
we can apply an equivalent transformation to the parameter vector $\wSrc$ instead of the features $u$. This results in added efficiency when the number of classifiers is small but the number and dimensionality of target examples is very high.

Since correlation alignment changes the features only, it can be applied to any base classifier. Due to its efficiency, it can also be especially advantageous when the target domains are changing rapidly, e.g., due to scene changes over the course of a long video stream. 

\subsection{Relationship to Existing Methods}
\paragraph{Relationship to Feature Normalization} 
It has long been known that input feature normalization improves many machine learning methods, e.g.,~\cite{batchnorm}. However, CORAL does not simply perform feature normalization, but rather aligns two different distributions. Standard feature normalization (zero mean and unit variance) does not address this issue, as illustrated in Figure~\ref{fig:variance}(a). In this example, although the features are normalized to have zero mean and unit variance in each dimension, the differences in correlations present in the source and target domains cause the distributions to be different. 

\paragraph{Relationship to Manifold Methods} 
 Recent state-of-the-art unsupervised approaches project the source and target distributions into a lower-dimensional manifold and find a transformation that brings the subspaces closer together~\cite{gopalan-iccv11,gfk,sasb,outlooks}. CORAL avoids subspace projection, which can be costly and requires selecting the hyper-parameter that controls the dimensionality of the subspace. We note that subspace-mapping approaches~\cite{outlooks,sasb} only align the top $k$ (subspace dimensionality) eigenvectors of the source and target covariance matrices. On the contrary, CORAL aligns the covariance matrices, which can only be re-constructed using all eigenvectors and eigenvalues. Even though the eigenvectors can be aligned well, the distributions can still differ a lot due to the difference of eigenvalues between the corresponding eigenvectors of the source and target data. CORAL is a more general and much simpler method than the above two as it takes into account~\emph{both} eigenvectors~\emph{and} eigenvalues of the covariance matrix~\emph{without} the burden of subspace dimensionality selection. 

\paragraph{Relationship to MMD methods}
Maximum Mean Discrepancy (MMD) based methods (e.g.,TCA~\cite{tca}, DAN~\cite{dan_long15}) for domain adaptation can be interpreted as ``moment matching'' and can express arbitrary statistics of the data. Minimizing MMD with polynomial kernel ($k(x,y) = (1+x'y)^d$ with $d=2$) is similar to the CORAL objective, however, no previous work has used this kernel for domain adaptation nor proposed a closed form solution to the best of our knowledge. 
The other difference is that MMD based approaches usually apply the~\emph{same} transformation to both the source and target domain.  As demonstrated in~\cite{ref:kulis_cvpr11,outlooks,sasb}, asymmetric transformations are more flexible and often yield better performance for domain adaptation tasks. Intuitively, symmetric transformations find a space that ``ignores'' the differences between the source and target domain while asymmetric transformations try to ``bridge'' the two domains.

\subsection{Application to Deep Neural Networks} 
Suppose $\phi(I)$ was computed by a multilayer neural network, then the inputs to each layer $\phi_k$ can suffer from covariate shift as well. Batch Normalization~\cite{batchnorm} tries to compensate for \emph{internal} covariate shift by normalizing each mini-batch to be zero-mean and unit-variance. However, as illustrated in Figure~\ref{fig:variance}, such normalization might not be enough. Even if used with full whitening, Batch Normalization may not compensate for \emph{external} covariate shift: the layer activations will be decorrelated for a source point but not for a target point. What's more, as mentioned in Section~\ref{subsec:algo}, whitening both domains still does not work. Our method can be easily integrated into a deep architecture by treating layers as features (e.g., fc6 or fc7 of AlexNet~\cite{alexnet}). Although we experiment only with CORAL applied to one hidden layer at each time, multilayer CORAL could be used by implementing the transformations $A_l$ as extra layers which follow each original layer $l$.
\section{Experiments}
\label{sec:exp}

We evaluate our method on object recognition~\cite{saenko2010adapting} and sentiment analysis~\cite{Blitzer07Biographies} with both shallow and deep features, using standard benchmarks and protocols. In all experiments we assume the target domain is unlabeled. 

We follow the standard procedure~\cite{sasb,decaf} and use a linear SVM as the base classifier. The model selection approach of~\cite{sasb} is used to set the $C$ parameter for the SVM by doing cross-validation on the source domain. Since there are no other hyperparameters (except the common regularization parameter $\lambda$ for whitening and coloring, which we discussed in Section~\ref{subsec:algo} and Figure~\ref{fig:sens}) required for our method, the results in this paper can be easily reproduced. To compare to published methods, we use the accuracies reported by their authors or conduct experiments using the source code provided by the authors. 

\subsection{Object Recognition}
\label{subsec:recog}
In this set of experiments, domain adaptation is used to improve the accuracy of an object classifier on novel image domains.
Both the standard Office~\cite{saenko2010adapting} and extended Office-Caltech10~\cite{gfk} datasets are used as benchmarks in this paper. Office-Caltech10 contains 10 object categories from an office environment (e.g., keyboard, laptop, etc.) in 4 image domains: $Webcam$, $DSLR$, $Amazon$, and $Caltech256$. The standard Office dataset contains 31 (the same 10 categories from Office-Caltech10 plus 21 additional ones) object categories in 3 domains: $Webcam$, $DSLR$, and $Amazon$. Later, we also conduct a larger (more data and categories) scale evaluation on Office-Caltech10 and the Cross-Dataset Testbed~\cite{cross_dataset} dataset. 

\begin{table*}\Large
\centering
\resizebox{1.5\columnwidth}{!}{
\begin{tabular}{|l||c|c|c|c|c|c|c|c|c|c|c|c|c|}
\hline
~ & A$\rightarrow$C & A$\rightarrow$D & A$\rightarrow$W & C$\rightarrow$A & C$\rightarrow$D & C$\rightarrow$W & D$\rightarrow$A & D$\rightarrow$C & D$\rightarrow$W & W$\rightarrow$A & W$\rightarrow$C & W$\rightarrow$D & AVG\\ 
\hline
NA & 35.8 & 33.1 & 24.9  &43.7 & 39.4 & 30.0 & 26.4 & 27.1  & 56.4 & 32.3 & 25.7 & 78.9 & 37.8\\ 
\hline
SVMA & 34.8 & 34.1  & 32.5 & 39.1 &  34.5 & 32.9 & 33.4  & 31.4  & 74.4 & 36.6 & 33.5 & 75.0 & 41.0\\ 
\hline
DAM & 34.9 & 34.3  & 32.5 & 39.2 &  34.7 & 33.1 & 33.5  & 31.5  & 74.7 & 34.7 & 31.2 & 68.3 & 40.2\\ 
\hline
GFK & 38.3 & 37.9  & 39.8&44.8 &  36.1 & 34.9 & 37.9  & 31.4  & 79.1 & 37.1 & 29.1 & 74.6 & 43.4\\ 
\hline
TCA & 40.0 & \textbf{39.1}  & \textbf{40.1}  & 46.7 &  \textbf{41.4} & 36.2 & 39.6  & 34.0  & 80.4 & \textbf{40.2} & 33.7 & 77.5 & 45.7\\ 
\hline
SA & 39.9 & 38.8  & 39.6  & 46.1 & 39.4 & 38.9 & \textbf{42.0}  & \textbf{35.0}  & 82.3 & 39.3 & 31.8 & 77.9 & 45.9\\ 
\hline
CORAL & \textbf{40.3} & 38.3 & 38.7 & \textbf{47.2} & 40.7 &\textbf{39.2} & 38.1 & 34.2 & \textbf{85.9} & 37.8 &\textbf{ 34.6} &\textbf{84.9} & \textbf{46.7}\\ 
\hline
\end{tabular}
}
\caption{\small Object recognition accuracies of all 12 domain shifts on the Office-Caltech10 dataset~\cite{gfk} with SURF features, following the protocol of~\cite{gfk,sasb,gopalan-iccv11,ref:kulis_cvpr11,saenko2010adapting}.}
\label{tab:result}
\vspace{-0.1in}
\end{table*}

\subsubsection{Object Recognition with Shallow Features}  
We follow the standard protocol of~\cite{gfk,sasb,gopalan-iccv11,ref:kulis_cvpr11,saenko2010adapting} and conduct experiments on the Office-Caltech10 dataset with shallow features (SURF).
The SURF features were encoded with 800-bin bag-of-words histograms and normalized to have zero mean and unit standard deviation in each dimension.  Since there are four domains, there are 12 experiment settings, namely, A$\rightarrow$C (train classifier on (A)mazon, test on (C)altech), A$\rightarrow$D (train on (A)mazon, test on (D)SLR), A$\rightarrow$W, and so on. We follow the standard protocol and conduct experiments in 20 randomized trials for each domain shift and average the accuracy over the trials. In each trial, we use the standard setting~\cite{gfk,sasb,gopalan-iccv11,ref:kulis_cvpr11,saenko2010adapting} and randomly sample the same number (20 for $Amazon$, $Caltech$, and $Webcam$; 8 for $DSLR$ as there are only 8 images per category in the $DSLR$ domain) of labelled images in the source domain as training set, and use all the unlabelled data in the target domain as the test set. 

In Table~\ref{tab:result}, we compare our method to five recent published methods: SVMA~\cite{svma}, DAM~\cite{ref:duan09}, GFK~\cite{gfk}, SA~\cite{sasb}, and TCA~\cite{tca} as well as the no adaptation baseline (NA). GFK, SA, and TCA are manifold based methods that project the source and target distributions into a lower-dimensional manifold. GFK integrates over an infinite number of subspaces along the subspace manifold using the kernel trick. SA aligns the source and target subspaces by computing a linear map that minimizes the Frobenius norm of their difference. TCA performs domain adaptation via a new parametric kernel using feature extraction methods by projecting data onto the learned transfer components. DAM introduces smoothness assumption to enforce the target classifier share similar decision values with the source classifiers. Even though these methods are far more complicated than ours and require tuning of hyperparameters (e.g., subspace dimensionality), our method achieves the best average performance across all the 12 domain shifts. Our method also improves on the no adaptation baseline (NA), in some cases increasing accuracy significantly (from 56\% to 86\% for D$\rightarrow$W). 

\subsubsection{Object Recognition with Deep Features}  
We follow the standard protocol of~\cite{decaf,tzeng_arxiv15,dan_long15,reversegrad} and conduct experiments on the standard Office dataset with deep features. DLID~\cite{chopra2013dlid} trains a joint source and target CNN architecture with an ``interpolating path'' between the source and target domain. DANN~\cite{DANN} incorporates the Maximum Mean Discrepancy (MMD) measure as a regularization to reduce the distribution mismatch. DA-NBNN~\cite{da_nbnn} presents an NBNN-based domain adaptation algorithm that  iteratively learns a class metric while inducing a large margin separation among classes. DECAF~\cite{decaf} uses  AlexNet~\cite{alexnet} pre-trained on ImageNet~\cite{imagenet} and extracts the fc6 or fc7 layers in the source domains as features to train a classifier. It then applies the classifier to the target domain directly. DDC~\cite{tzeng_arxiv15} adds a domain confusion loss to AlexNet~\cite{alexnet} and fine-tunes it on both the source and target domain.  

DAN~\cite{dan_long15} and ReverseGrad~\cite{reversegrad} are the two most recent domain adaptation approaches based on deep architectures. DAN is similar to DDC but utilizes a multi-kernel selection method for better mean embedding matching and adapts in multiple layers. ReverseGrad introduces a gradient reversal layer to allow direct optimization through back-propagation. Both DDC and ReverseGrad add a new binary classification task by treating the source and target domain as two classes. They maximize the binary classification loss to obtain invariant features. 

To have a fair comparison, we apply CORAL to both the pre-trained AlexNet (CORAL-fc6 and CORAL-fc7) and to AlexNet fine-tuned on the source (CORAL-FT6 and CORAL-FT7). However, the fine-tuning procedures of DDC, DAN, and ReverseGrad are very complicated as there is more than one loss and hyper-parameters are needed to combine them. They also require adding new layers and data from both source and target domains. We use standard fine-tuning on the source domain only to get the baseline NA results (NA-FT6 and NA-FT7). Since there are three domains, there are 6 experiment settings. We follow the protocol of~\cite{decaf,tzeng_arxiv15,reversegrad} and conduct experiments on 5 random training/test splits and get the mean accuracy for each domain shift. 

\begin{table}[t]
\begin{center}
\resizebox{\columnwidth}{!}{
\begin{tabular}{|l||c|c|c|c|c|c|c|}
\hline
~ & A$\rightarrow$D&	A$\rightarrow$W	&D$\rightarrow$A&	D$\rightarrow$W	&W$\rightarrow$A	&W$\rightarrow$D&AVG\\ 
\hline
NA-fc6	&53.2	&48.6	&40.5	&92.9	&39.0	&98.8 &62.2\\
\hline
NA-fc7	&55.7	&50.6	&46.5	&93.1	&43.0	&97.4&64.4\\
\hline
NA-FT6	&54.5	&48.0	&38.9	&91.2	&40.7	&98.9&62.0\\
\hline
NA-FT7	&58.5	&53.0	&43.8	&94.8	&43.7	&99.1&65.5\\
\hline
\hline
SA-fc6	&41.3	&35	&32.3	&74.5	&30.1	&81.5 &49.1 \\
\hline
SA-fc7	&46.2	&42.5	&39.3	&78.9	&36.3	&80.6 &54.0 \\
\hline
SA-FT6	&40.5	&41.1	&33.8	&85.4	&33.4	&88.2 &53.7 \\
\hline
SA-FT7	&50.5	&47.2	&39.6	&89	       &37.3	        &93 &59.4 \\
\hline
\hline
GFK-fc6	&44.8	&37.8	&34.8	&81	&31.4	&86.9 &49.1 \\
\hline
GFK-fc7	&52	&48.2	&41.8	&86.5	&38.6	&87.5 &59.1 \\
\hline
GFK-FT6	&48.8	&45.6	&40.5	&90.4	&36.7	&96.3 &59.7 \\
\hline
GFK-FT7	&56.4	&52.3	&43.2	&92.2	&41.5	&96.6 &63.7 \\
\hline
TCA-fc6	&40.6	&36.8	&32.9	&82.3	&28.9	&84.1 &50.9 \\
\hline
TCA-fc7	&45.4	&40.5	&36.5	&78.2	&34.1	&84 &53.1 \\
\hline
TCA-FT6	&40.8	&37.2	&30.6	&79.5	&36.7	&91.8 &52.8 \\
\hline
TCA-FT7	&47.3	&45.2	&36.4	&80.9	&39.2	&92 &56.8 \\
\hline
\hline
DLID	&-&	26.1&-&		68.9&-&84.9&-\\ 
\hline
DANN &34.0	&34.1&	20.1&	62.0&	21.2&	64.4&39.3\\
\hline
DA-NBNN	&-	&23.3&	-&	67.2&	-&	67.4&-\\
\hline
DECAF-fc6  &-	&52.2&	-&	91.5&	-&	-&-\\
\hline
DECAF-fc7  &-	&53.9&	-&	89.2&	-&	-&-\\
\hline
DDC	 &-	&59.4	&-	&92.5	&-	&91.7  &- \\
\hline
DAN	 &-	&66.0	&-	&93.5	&-	&95.3  &- \\
\hline
ReverseGrad	 &-	&\textbf{67.3}	&-	&94.0	&-	&93.7  &- \\
\hline
\hline
CORAL-fc6	&53.7	&48.4	&44.4	&96.5	&41.9	&99.2&64.0\\
\hline
CORAL-fc7	&57.1	&53.1	&\textbf{51.1}	&94.6	&47.3	&98.2&66.9\\
\hline
CORAL-FT6	&61.2	&59.8	&47.4	&\textbf{97.1}	&45.8	&\textbf{99.5}&68.5\\
\hline
CORAL-FT7	&\textbf{62.2}	&61.9	&48.4	&96.2	&\textbf{48.2}	&\textbf{99.5}&\textbf{69.4}\\
\hline
\end{tabular}
}
\end{center}
\caption{\small Object recognition accuracies of all 6 domain shifts on the standard Office dataset~\cite{saenko2010adapting} with deep features, following the protocol of~\cite{decaf,tzeng_arxiv15,reversegrad}. }
\label{tab:result_office10}
\vspace{-0.1in}
\end{table}

\begin{table*}\small
\centering
\resizebox{1.8\columnwidth}{!}{
\begin{tabular}{|l||c|c|c|c|c|c|c|c|c|c|c|c|c|}
\hline
~ & A$\rightarrow$C & A$\rightarrow$D & A$\rightarrow$W & C$\rightarrow$A & C$\rightarrow$D & C$\rightarrow$W & D$\rightarrow$A & D$\rightarrow$C & D$\rightarrow$W & W$\rightarrow$A & W$\rightarrow$C & W$\rightarrow$D & AVG\\ 
\hline
NA & 41.7 & \textbf{44.6} & 31.9  & 53.1 & 47.8 & 41.7 & 26.2 & 26.4 & 52.5 & 27.6 & 21.2 & 78.3 & 41.1\\ 
\hline
SA & 37.4 & 36.3 & 39.0  & 44.9 & 39.5 & 41.0 & 32.9 & 34.3 & 65.1 & 34.4 & 31.0 & 62.4 & 41.5\\ 
\hline
GFK &41.9 & 41.4 & 41.4 & \textbf{56.0}  & 42.7 & 45.1 & \textbf{38.7} & \textbf{36.5} & 74.6 & 31.9 & 27.5 & 79.6 &46.4\\ 
\hline
TCA &35.2 & 39.5 & 29.5  & 46.8 & \textbf{52.2} & 38.6 & 36.2 & 30.1 & 71.2 & 32.2 & 27.9 & 74.5 &42.8\\ 
\hline
CORAL & \textbf{45.1} & 39.5 & \textbf{44.4}  & 52.1 & 45.9 & \textbf{46.4}  & 37.7 & 33.8 & \textbf{84.7} & \textbf{36.0} & \textbf{33.7} & \textbf{86.6} & \textbf{48.8}\\ 
\hline
\end{tabular}
}
\caption{\small Object recognition accuracies of all 12 domain shifts on the Office-Caltech10 dataset~\cite{gfk} with SURF features, using the~\emph{``full training''} protocol.}
\label{tab:result_office10_large}
\vspace{-0.1in}
\end{table*}

\begin{table}\small
\begin{center}
\resizebox{0.95\columnwidth}{!}{
\begin{tabular}{|l||c|c|c|c|c|c|c|}
\hline
~ & C$\rightarrow$I&	C$\rightarrow$S	&I$\rightarrow$C&	I$\rightarrow$S	&S$\rightarrow$C	&S$\rightarrow$I&AVG\\ 
\hline
NA	&66.1	 &21.9	 &73.8	&22.4	 &24.6	&22.4  & 38.5\\
\hline
SA	&43.7	&13.9	&52.0	&15.1	&15.8	&14.3  &25.8\\
\hline
GFK	&52   	&18.6        &58.5	&20.1	&21.1       &17.4      &31.3\\
\hline
TCA	&48.6	&15.6	&54.0   	&14.8	&14.6	&12.0      &26.6\\
\hline
CORAL &\textbf{66.2}	&\textbf{22.9}	&\textbf{74.7}	&\textbf{25.4}	&\textbf{26.9} 	&\textbf{25.2}     &\textbf{40.2}\\
\hline
\end{tabular}
}
\end{center}
\caption{\small Object recognition accuracies of all 6 domain shifts on the Testbed Cross-Dataset~\cite{cross_dataset} dataset with DECAF-fc7 features, using the~\emph{``full training''} protocol.}
\label{tab:result_cross}
\vspace{-0.1in}
\end{table}

\begin{table}[t]
\centering
\resizebox{0.7\columnwidth}{!}{
\begin{tabular}{|l||c|c|c|c|c|}
\hline
~ & K$\rightarrow$D&	D$\rightarrow$B	& B$\rightarrow$E&	E$\rightarrow$K &AVG\\ 
\hline
NA	&72.2 & 76.9 & 74.7 & 82.8& 76.7\\ 
\hline
TCA	&60.4 & 61.4 & 61.3 & 68.7& 63.0\\ 
\hline
SA	&\textbf{78.4} & 74.7 & 75.6 & 79.3& 77.0\\ 
\hline
GFS &67.9 & 68.6 & 66.9 & 75.1& 69.6\\ 
\hline
GFK &69.0 & 71.3 & 68.4 & 78.2& 71.7\\ 
\hline
SCL	&72.8 & 76.2 & 75.0 & 82.9 & 76.7\\ 
\hline
KMM &72.2 & \textbf{78.6} & \textbf{76.9} & 83.5  & 77.8\\ 
\hline
CORAL &73.9 & 78.3 &76.3 & \textbf{83.6}& \textbf{78.0}\\ 
\hline
\end{tabular}
}
\caption{\small Review classification accuracies of the 4 standard domain shifts~\cite{gong-icml13} on the Amazon dataset~\cite{Blitzer07Biographies} with bag-of-words features.}
\label{tab:result_nlp}
\vspace{-0.1in}
\end{table}

In Table~\ref{tab:result_office10} we compare our method to the 11 baseline methods discussed before. Again, our method outperforms all of these techniques in almost all cases, sometimes by a very large margin. Note that most of the deep structures based methods report results only on some settings. We find that the higher level fc7/FT7 features lead to better performance than fc6/FT6. What's more, the NA baselines also achieve very good performance, even better than all the manifold methods and some deep methods. However, CORAL outperforms it consistently and is the only method achieves better AVG performance across all the 6 shifts. It also achieves better peformance than the two latest deep methods (DAN and ReverseGrad) in 2 out of the 3 shifts they reported.

One interesting finding is that, although fine-tuning on the source domain only (NA-FT6 and NA-FT7) does not achieve better performance on the target domain compared to the pre-trained network (NA-fc6 and NA-fc7), applying CORAL to the fine-tuned network (CORAL-FT6 and CORAL-FT7) achieves much better performance than applying CORAL to the pre-trained network (CORAL-fc6 and CORAL-fc7). One possible explanation is that the pre-trained network might be underfitting while the fine-tuned network is overfitting. Since CORAL aligns the source feature distribution to target distribution, overfitting becomes less of a problem. 

\subsubsection{A Larger Scale Evaluation}
\label{subsubsec:larger}
In this section, we repeat the evaluation on a larger scale. We conduct two sets of experiments to investigate how the dataset size and number of classes will affect the performance of domain adaptation methods. In both sets of experiments, we use the~\emph{``full training''} protocol, where all the source data are used for training, compared to the standard subsampling protocol in the previous two sections. Since all the target data are used in the previous two sections, the only difference between these two settings is the training dataset size of the source domain. To have a direct comparison to Table~\ref{tab:result}, we conduct the first set of experiments on the Office-Caltech10 dataset with SURF features. To investigate the effect of the number of classes, we conduct the second set of experiments on the Cross-Dataset Testbed~\cite{cross_dataset} dataset, with 3847 images for Caltech256~\cite{caltech256}, 4000 images for ImageNet~\cite{imagenet}, and 2626 images for SUN~\cite{sun_data} over 40 classes, using the only publicly available deep features (DECAF-fc7). 

In Tables~\ref{tab:result_office10_large} and~\ref{tab:result_cross}, we compare CORAL to SA, GFK, TCA which have available source code as well as the NA baseline. Table~\ref{tab:result_office10_large} shows the result of the Office-Caltech10 dataset and Table~\ref{tab:result_cross} shows the result on the Cross-Dataset Testbed dataset. In both experiments, CORAL outperforms all the baseline methods and again the margin on deep features is much larger than on shallow features. Comparing Table~\ref{tab:result_office10_large} to Table~\ref{tab:result}, we can say that the performance difference between NA and other methods is smaller as more source data is used. This may be due to the fact that as more training data is used, the classifier is stronger and can generalize better to other domains.

\subsection{Sentiment Analysis}
We also evaluate our method on sentiment analysis using the standard Amazon review dataset~\cite{Blitzer07Biographies,gong-icml13}. We use the processed data from~\cite{gong-icml13}, in which the dimensionality of the bag-of-words features was reduced to keep the top 400 words without losing performance. This dataset contains Amazon reviews on 4 domains: Kitchen appliances, DVD, Books, and Electronics. For each domain, there are 1000 positive and 1000 negative reviews. We follow the standard protocol of~\cite{gong-icml13} and conduct experiments on 20 random training/test splits and report the mean accuracy for each domain shift. 

In Table~\ref{tab:result_nlp}, we compare our method to five 
published methods: TCA~\cite{tca}, GFS~\cite{gopalan-iccv11}, GFK~\cite{gfk}, SCL~\cite{scl}, and KMM~\cite{huang_nips06} as well as the no adaptation baseline (NA). GFS is a precursor of GFK and interpolates features using a finite number of subspaces. SCL introduces structural correspondence learning to automatically induce correspondences among features from different domains. KMM presents a nonparametric method to directly produce re-sampling weights without distribution estimation. One interesting observation is that, for this sentiment analysis task, three state-of-the-art methods (TCA, GFS, and GFK) actually perform worse than the no adaptation baseline (NA). Despite the difficulty of this task, CORAL still performs well and achieves the best average classification accuracy across the 4 standard domain shifts.

\section{Discussion}
One interesting result is that the margin between CORAL and other published methods is much larger on deep features (e.g. 64.0 of CORAL-fc6 compared to 49.1 of SA-fc6 in Table~\ref{tab:result_office10}) than on bag-of-words features. This could be because deep features are more strongly correlated than bag-of-words features (e.g. the largest singular value of the covariance matrix of Amazon-fc6 is 354 compared to 27 of Amazon-SURF). Similarly, the improvement on images (Tables 1-4) is much larger than text (Table~\ref{tab:result_nlp}), possibly because bag-of-words text features are extremely sparse and less correlated than image features.
As demonstrated in~\cite{deep_vis}, high level deep features are more ``parts'' or ``objects'. Intuitively, ``parts'' or ``objects'' should be more strongly correlated than ``edges'' (e.g., arm and head of a person are more likely to appear jointly).

These findings suggest that CORAL is extremely valuable in the era of deep learning. Applying CORAL to deep text features is part of future work.
\section{Conclusion}
\label{sec:concl}
In this article, we proposed an simple, efficient and effective method for domain adaptation. The method is ``frustratingly easy'' to implement: the only computation involved is re-coloring the whitened source features with the covariance of the target domain. 

Extensive experiments on standard benchmarks demonstrate the superiority of our method over many existing state-of-the-art methods. These results confirm that CORAL is applicable to multiple features types, including highly-performing deep features, and to different tasks, including computer vision and natural language processing. 
\section{Acknowledgments}
\label{sec:ackn}
The authors would like to thank Mingsheng Long, Judy Hoffman, and Trevor Darrell for helpful discussions and suggestions; the reviewers for their valuable comments. The Tesla K40 used for this research was donated by the NVIDIA Corporation. This research was supported by NSF Awards IIS-1451244 and IIS-1212928. 
\bibliography{main}

\begin{thebibliography}{}

\bibitem[\protect\citeauthoryear{Ben-David \bgroup et al\mbox.\egroup
  }{2007}]{bendavid}
Ben-David, S.; Blitzer, J.; Crammer, K.; and Pereira, F.
\newblock 2007.
\newblock Analysis of representations for domain adaptation.
\newblock In {\em NIPS}.

\bibitem[\protect\citeauthoryear{Blitzer, Dredze, and
  Pereira}{2007}]{Blitzer07Biographies}
Blitzer, J.; Dredze, M.; and Pereira, F.
\newblock 2007.
\newblock {Biographies, Bollywood, Boom-boxes and Blenders: Domain Adaptation
  for Sentiment Classification}.
\newblock In {\em ACL}.

\bibitem[\protect\citeauthoryear{Blitzer, McDonald, and Pereira}{2006}]{scl}
Blitzer, J.; McDonald, R.; and Pereira, F.
\newblock 2006.
\newblock Domain adaptation with structural correspondence learning.
\newblock In {\em EMNLP}.

\bibitem[\protect\citeauthoryear{Borgwardt \bgroup et al\mbox.\egroup
  }{2006}]{mmd}
Borgwardt, K.~M.; Gretton, A.; Rasch, M.~J.; Kriegel, H.-P.; Sch{\"o}lkopf, B.;
  and Smola, A.~J.
\newblock 2006.
\newblock Integrating structured biological data by kernel maximum mean
  discrepancy.
\newblock In {\em Bioinformatics}.

\bibitem[\protect\citeauthoryear{Cai, Cand\`{e}s, and Shen}{2010}]{SVT}
Cai, J.-F.; Cand\`{e}s, E.~J.; and Shen, Z.
\newblock 2010.
\newblock A singular value thresholding algorithm for matrix completion.
\newblock {\em SIAM J. on Optimization} 20(4):1956--1982.

\bibitem[\protect\citeauthoryear{Caseiro \bgroup et al\mbox.\egroup
  }{2015}]{Sp_CVPR15}
Caseiro, R.; Henriques, J.~F.; Martins, P.; and Batista, J.
\newblock 2015.
\newblock Beyond the shortest path : Unsupervised domain adaptation by sampling
  subspaces along the spline flow.
\newblock In {\em CVPR}.

\bibitem[\protect\citeauthoryear{Chopra, Balakrishnan, and
  Gopalan}{2013}]{chopra2013dlid}
Chopra, S.; Balakrishnan, S.; and Gopalan, R.
\newblock 2013.
\newblock Dlid: Deep learning for domain adaptation by interpolating between
  domains.
\newblock In {\em ICML Workshop}.

\bibitem[\protect\citeauthoryear{{Daume~III}}{2007}]{daume}
{Daume~III}, H.
\newblock 2007.
\newblock Frustratingly easy domain adaptation.
\newblock In {\em ACL}.

\bibitem[\protect\citeauthoryear{Deng \bgroup et al\mbox.\egroup
  }{2009}]{imagenet}
Deng, J.; Dong, W.; Socher, R.; Li, L.-J.; Li, K.; and Fei-Fei, L.
\newblock 2009.
\newblock Imagenet: A large-scale hierarchical image database.
\newblock In {\em CVPR}.

\bibitem[\protect\citeauthoryear{Donahue \bgroup et al\mbox.\egroup
  }{2014}]{decaf}
Donahue, J.; Jia, Y.; Vinyals, O.; Hoffman, J.; Zhang, N.; Tzeng, E.; and
  Darrell, T.
\newblock 2014.
\newblock Decaf: A deep convolutional activation feature for generic visual
  recognition.
\newblock In {\em ICML}.

\bibitem[\protect\citeauthoryear{Duan \bgroup et al\mbox.\egroup
  }{2009}]{ref:duan09}
Duan, L.; Tsang, I.~W.; Xu, D.; and Chua, T.
\newblock 2009.
\newblock Domain adaptation from multiple sources via auxiliary classifiers.
\newblock In {\em ICML}.

\bibitem[\protect\citeauthoryear{Duan, Tsang, and Xu}{2012}]{svma}
Duan, L.; Tsang, I.~W.; and Xu, D.
\newblock 2012.
\newblock Domain transfer multiple kernel learning.
\newblock {\em TPAMI} 34(3):465--479.

\bibitem[\protect\citeauthoryear{Fernando \bgroup et al\mbox.\egroup
  }{2013}]{sasb}
Fernando, B.; Habrard, A.; Sebban, M.; and Tuytelaars, T.
\newblock 2013.
\newblock Unsupervised visual domain adaptation using subspace alignment.
\newblock In {\em ICCV}.

\bibitem[\protect\citeauthoryear{Ganin and Lempitsky}{2015}]{reversegrad}
Ganin, Y., and Lempitsky, V.
\newblock 2015.
\newblock Unsupervised domain adaptation by backpropagation.
\newblock In {\em ICML}.

\bibitem[\protect\citeauthoryear{Ghifary, Kleijn, and Zhang}{2014}]{DANN}
Ghifary, M.; Kleijn, W.~B.; and Zhang, M.
\newblock 2014.
\newblock Domain adaptive neural networks for object recognition.
\newblock In {\em PRICAI}.

\bibitem[\protect\citeauthoryear{Gong \bgroup et al\mbox.\egroup }{2012}]{gfk}
Gong, B.; Shi, Y.; Sha, F.; and Grauman, K.
\newblock 2012.
\newblock Geodesic flow kernel for unsupervised domain adaptation.
\newblock In {\em CVPR}.

\bibitem[\protect\citeauthoryear{Gong, Grauman, and Sha}{2013}]{gong-icml13}
Gong, B.; Grauman, K.; and Sha, F.
\newblock 2013.
\newblock Connecting the dots with landmarks: Discriminatively learning
  domain-invariant features for unsupervised domain adaptation.
\newblock In {\em ICML}.

\bibitem[\protect\citeauthoryear{Gopalan, Li, and
  Chellappa}{2011}]{gopalan-iccv11}
Gopalan, R.; Li, R.; and Chellappa, R.
\newblock 2011.
\newblock Domain adaptation for object recognition: An unsupervised approach.
\newblock In {\em ICCV}.

\bibitem[\protect\citeauthoryear{Gregory, Alex, and Pietro}{2007}]{caltech256}
Gregory, G.; Alex, H.; and Pietro, P.
\newblock 2007.
\newblock Caltech 256 object category dataset.
\newblock In {\em Tech. Rep. UCB/CSD-04-1366, California Institue of
  Technology}.

\bibitem[\protect\citeauthoryear{Harel and Mannor}{2011}]{outlooks}
Harel, M., and Mannor, S.
\newblock 2011.
\newblock Learning from multiple outlooks.
\newblock In {\em ICML}.

\bibitem[\protect\citeauthoryear{Hariharan, Malik, and Ramanan}{2012}]{who}
Hariharan, B.; Malik, J.; and Ramanan, D.
\newblock 2012.
\newblock Discriminative decorrelation for clustering and classification.
\newblock In {\em ECCV}.

\bibitem[\protect\citeauthoryear{Huang and Wang}{2013}]{dict_2}
Huang, D.-A., and Wang, Y.-C.
\newblock 2013.
\newblock Coupled dictionary and feature space learning with applications to
  cross-domain image synthesis and recognition.
\newblock In {\em ICCV}.

\bibitem[\protect\citeauthoryear{Huang \bgroup et al\mbox.\egroup
  }{2006}]{huang_nips06}
Huang, J.; Smola, A.~J.; Gretton, A.; Borgwardt, K.~M.; and Sch{\"o}lkopf, B.
\newblock 2006.
\newblock Correcting sample selection bias by unlabeled data.
\newblock In {\em NIPS}.

\bibitem[\protect\citeauthoryear{Ioffe and Szegedy}{2015}]{batchnorm}
Ioffe, S., and Szegedy, C.
\newblock 2015.
\newblock Batch normalization: Accelerating deep network training by reducing
  internal covariate shift.
\newblock In {\em ICML}.

\bibitem[\protect\citeauthoryear{Jiang and Zhai}{2007}]{jiang-zhai07}
Jiang, J., and Zhai, C.
\newblock 2007.
\newblock {Instance Weighting for Domain Adaptation in NLP}.
\newblock In {\em ACL}.

\bibitem[\protect\citeauthoryear{Krizhevsky, Sutskever, and
  Hinton}{2012}]{alexnet}
Krizhevsky, A.; Sutskever, I.; and Hinton, G.~E.
\newblock 2012.
\newblock Imagenet classification with deep convolutional neural networks.
\newblock In {\em NIPS}.

\bibitem[\protect\citeauthoryear{Kulis, Saenko, and
  Darrell}{2011}]{ref:kulis_cvpr11}
Kulis, B.; Saenko, K.; and Darrell, T.
\newblock 2011.
\newblock What you saw is not what you get: Domain adaptation using asymmetric
  kernel transforms.
\newblock In {\em CVPR}.

\bibitem[\protect\citeauthoryear{Long \bgroup et al\mbox.\egroup
  }{2014}]{long_cvpr}
Long, M.; Wang, J.; Ding, G.; Sun, J.; and Yu, P.
\newblock 2014.
\newblock Transfer joint matching for unsupervised domain adaptation.
\newblock In {\em CVPR}.

\bibitem[\protect\citeauthoryear{Long \bgroup et al\mbox.\egroup
  }{2015}]{dan_long15}
Long, M.; Cao, Y.; Wang, J.; and Jordan, M.~I.
\newblock 2015.
\newblock Learning transferable features with deep adaptation networks.
\newblock In {\em ICML}.

\bibitem[\protect\citeauthoryear{Mahendran and Vedaldi}{2015}]{deep_vis}
Mahendran, A., and Vedaldi, A.
\newblock 2015.
\newblock Understanding deep image representations by inverting them.
\newblock In {\em CVPR}.

\bibitem[\protect\citeauthoryear{Pan \bgroup et al\mbox.\egroup }{2009}]{tca}
Pan, S.~J.; Tsang, I.~W.; Kwok, J.~T.; and Yang, Q.
\newblock 2009.
\newblock Domain adaptation via transfer component analysis.
\newblock In {\em IJCAI}.

\bibitem[\protect\citeauthoryear{Saenko \bgroup et al\mbox.\egroup
  }{2010}]{saenko2010adapting}
Saenko, K.; Kulis, B.; Fritz, M.; and Darrell, T.
\newblock 2010.
\newblock Adapting visual category models to new domains.
\newblock In {\em ECCV}.

\bibitem[\protect\citeauthoryear{Shekhar \bgroup et al\mbox.\egroup
  }{2013}]{dict_1}
Shekhar, S.; Patel, V.~M.; Nguyen, H.~V.; and Chellappa, R.
\newblock 2013.
\newblock Generalized domain-adaptive dictionaries.
\newblock In {\em CVPR}.

\bibitem[\protect\citeauthoryear{Sun and Saenko}{2014}]{bmvc}
Sun, B., and Saenko, K.
\newblock 2014.
\newblock From virtual to reality: Fast adaptation of virtual object detectors
  to real domains.
\newblock In {\em BMVC}.

\bibitem[\protect\citeauthoryear{Tommasi and Caputo}{2013}]{da_nbnn}
Tommasi, T., and Caputo, B.
\newblock 2013.
\newblock Frustratingly easy {NBNN} domain adaptation.
\newblock In {\em ICCV}.

\bibitem[\protect\citeauthoryear{Tommasi and Tuytelaars}{2014}]{cross_dataset}
Tommasi, T., and Tuytelaars, T.
\newblock 2014.
\newblock A testbed for cross-dataset analysis.
\newblock In {\em ECCV TASK-CV Workshop}.

\bibitem[\protect\citeauthoryear{Torralba and Efros}{2011}]{efros-cvpr11}
Torralba, A., and Efros, A.~A.
\newblock 2011.
\newblock Unbiased look at dataset bias.
\newblock In {\em CVPR}.

\bibitem[\protect\citeauthoryear{Tzeng \bgroup et al\mbox.\egroup
  }{2014}]{tzeng_arxiv15}
Tzeng, E.; Hoffman, J.; Zhang, N.; Saenko, K.; and Darrell, T.
\newblock 2014.
\newblock Deep domain confusion: Maximizing for domain invariance.
\newblock {\em CoRR} abs/1412.3474.

\bibitem[\protect\citeauthoryear{Xiao \bgroup et al\mbox.\egroup
  }{2010}]{sun_data}
Xiao, J.; Hays, J.; Ehinger, K.~A.; Oliva, A.; and Torralba, A.
\newblock 2010.
\newblock Sun database: Large-scale scene recognition from abbey to zoo.
\newblock In {\em CVPR}.

\bibitem[\protect\citeauthoryear{Yang, Yan, and Hauptmann}{2007}]{yang_icdm07}
Yang, J.; Yan, R.; and Hauptmann, A.
\newblock 2007.
\newblock Adapting {SVM} classifiers to data with shifted distributions.
\newblock In {\em ICDM Workshop}.

\end{thebibliography}
\bibliographystyle{aaai}

\end{document}